\newif\ifarxiv
\arxivtrue

\documentclass[letterpaper,11pt]{article}
\usepackage{fullpage}
\usepackage[colorlinks,citecolor=blue]{hyperref}

\usepackage[color,secparam=$\lambda$]{crypto/crypto}
\usepackage{times}
\usepackage{amsthm}

\newtheorem{definition}{Definition}

\newtheorem{theorem}{Theorem}
\newtheorem{assumption}{Assumption}

\newtheorem{lemma}{Lemma}
\usepackage{caption,subcaption}
\usepackage{graphicx}
\usepackage{amsfonts}
\usepackage{float}
\usepackage{bm}
\usepackage{enumitem}
\usepackage{algorithm2e}
\usepackage{multirow}
\usepackage{amsmath}

\newif\ifusenix
\usenixtrue

\newcommand{\sharpness}{directional sharpness\xspace}
\newcommand{\Sharpness}{Directional sharpness\xspace}
\newcommand{\SHARPNESS}{Directional Sharpness\xspace}
\newcommand{\data}{\ensuremath{D}\xspace}

\author{
 {\rm Gefei Tan}\\
 Northwestern University \& Google \\
  {\tt\small gefeitan@u.northwestern.edu}
 \and 
 {\rm Adri\`a Gasc\'on}\\
 Google \\
 {\tt\small adriag@google.com}
 \and 
 {\rm Sarah Meiklejohn}\thanks{Portions of this work were done while the author was at Google.}\\
 University College London \\
 {\tt\small s.meiklejohn@ucl.ac.uk}
 \and 
 {\rm Mariana Raykova}\\
 Google\\
 {\tt\small marianar@google.com}
}

\begin{document}
\date{}
\ifarxiv
\title{Certification of Machine Learning Models via Directional Sharpness}
\else
\title{\Large \bf Certification of Machine Learning Models via Directional Sharpness}
\fi
\maketitle

\begin{abstract}
In machine learning, model certification has been identified as an important method for gaining assurance about a model's trustworthiness and quality. A model’s quality is largely determined by its ability to generalize, i.e., to perform well on data beyond what it was trained on. It is not possible to certify generalization directly, however, as it depends on unknown data and is not directly measurable. Proxies such as test accuracy can be misleading when the training process is perturbed (intentionally or accidentally), 
and metrics such as \emph{sharpness}---which has an empirically supported link to generalization---are computationally expensive and can also serve as unreliable signals when training deviates from a prescribed procedure. In this work, we propose \emph{directional sharpness}, a metric designed to efficiently and reliably indicate generalization despite potential training deviations. 
We provide empirical and analytical evidence that directional sharpness (1) correlates more strongly with generalization than existing metrics and (2) identifies models with poor generalization more reliably than existing metrics. Furthermore, directional sharpness is efficiently computable in model auditing settings, where the verifier has access to training data, and via zero-knowledge proofs that certify quality without revealing training data.

\end{abstract}

\section{Introduction}
\label{sec:intro}

There are many properties of machine learning (ML) models that model providers and external stakeholders may wish to verify before those models are deployed.
Arguably, the most universal property for such \emph{model certification} is the \emph{quality} of the model, in terms of its performance on downstream tasks.
A natural and widely accepted notion of model quality is \emph{generalization}~\cite{PACbook,gen_bound_book}: the ability of the model to perform well on new, unseen data.

In practice, model quality is typically estimated using a holdout test set, and a model is said to ``generalize well'' if it achieves high test accuracy. Test accuracy is not a reliable metric, however, as test sets can contain errors~\cite{northcutt2021pervasive}, fail to detect overfitting~\cite{dwork2015preserving,geirhos2020shortcut,gururangan2018annotation}, or miss real-world distribution shifts~\cite{recht2019imagenet,koh2021wilds,barz2020we,tampu2022inflation}.
In adversarial settings, research has shown that a model can achieve high accuracy on test data while being vulnerable to adversarial inputs~\cite{goodfellow2015explaining} or containing backdoors that activate only on specific triggers~\cite{badnet,9802938,gao2020backdoor}.

In statistical learning theory, generalization is captured by the gap between the empirical risk on the training set and population risk on fresh samples from the same (unknown) data distribution. Ideally, a certificate of generalization would show that this gap is small.
However, the true data distribution is unknown, and thus generalization cannot be measured directly. Researchers have instead focused on generalization bounds: upper bounds on this gap computable from the training set and model~\cite{PACbook,dlbound,pac1,pac2,gen_bound_book,itbound}. While theoretically appealing, existing bounds are often numerically vacuous or too loose to track actual performance for modern, over-parameterized neural networks, rendering them unsuitable for certification.

These gaps have motivated empirically grounded \textit{generalization metrics}: quantities computable from the learned model and training data that reliably predict generalization.
Such metrics can serve as a principled proxy for generalization without relying on a specific holdout test set.

Recently, the \emph{sharpness} of the loss landscape has emerged as a compelling generalization metric, with empirical and theoretical evidence that models in wide ``flat'' basins generalize better than those in narrow ``sharp'' minima~\cite{keskar_large-batch_2017, jiang_fantastic_2019,petzka_relative_2021,haddouche_pac-bayesian_2025,chen_why_2023}.
One particularly effective approach measures the \emph{sensitivity} of training loss to small parameter perturbations~\cite{sam,ASAM,keskar_large-batch_2017,petzka_relative_2021,Dziugaite_Roy_2017}.
The seminal work of Foret et al.~\cite{sam} introduces sharpness-aware minimization (SAM), an optimizer that explicitly minimizes the model's sharpness.
SAM has been shown to significantly improve generalization~\cite{sam,ASAM,chen_why_2023}, robustness to noisy labels~\cite{baek_why_2024}, and adversarial robustness~\cite{wu2020adversarial} across various benchmarks and model architectures. SAM training is also associated with improved resistance to membership inference attacks~\cite{liu_generalization_2021}, backdoor defense~\cite{zhu2023enhancing}, and improvements in machine unlearning~\cite{tang2025sharpnessawaremachineunlearning,fan_towards_2025}. These results provide strong evidence that perturbation-based sharpness is a promising proxy for model generalization.

Despite this explanatory power, existing sharpness metrics have not been explored in the context of certification, likely due to two fundamental barriers.
First, these metrics have been validated only via correlation with test accuracy on models trained honestly on clean data. Given that test accuracy itself can fail to detect poor generalization,
it remains unclear whether or not sharpness can reliably predict generalization under (intentional or unintentional) training deviations.
Second, existing metrics incur prohibitive computational cost even with direct access to the model and training data, let alone the overhead required in a setting where an external stakeholder would like to verify (without access to the training data) a cryptographic proof computed by the model provider.

This motivates our goal: to \emph{design a generalization metric that is strongly predictive of generalization, remains informative under training deviations, and is efficient enough to support privacy-preserving verification.}  

\smallskip\noindent\textbf{Our results.} In this paper, we introduce \emph{\sharpness}, a generalization metric designed specifically for model certification. Unlike static sharpness, which measures the loss landscape geometry at a single point, \sharpness measures \textit{dynamic stability}: the extent to which
a model remains in a flat region under continuous, stochastic perturbation. We show that
\sharpness(1) outperforms existing metrics in terms of empirical correlation with generalization; (2) remains predictive of generalization even under training deviations, making it more sensitive in detecting poor generalization than existing measures; and (3) is significantly more efficient to compute.
\Sharpness thus provides an efficient and robust metric for evaluating generalization when the verifier has access to the training data. We also show how to extend this metric to the setting where
the verifier does not have access to the training dataset, via zero-knowledge proofs for \sharpness.

\subsection{Technical Overview}
We now outline the key ideas behind \sharpness. We first explain the shift from static to dynamic measurement, then discuss its theoretical grounding, and finally present empirical validation showing the metric is efficient and enables practical certification.

\smallskip\noindent\textbf{From static sharpness to dynamic stability.}
Our starting point is perturbation-based sharpness, which shows strong empirical correlation with generalization in honestly trained models~\cite{jiang_fantastic_2019}.
However, certification demands more than correlation in benign settings: a metric must detect low-quality models that are adversarially chosen to evade detection.

We identify a fundamental limitation of existing sharpness metrics:  they are \emph{static}, and evaluate a single, averaged perturbation around the final model to measure sharpness. This one-shot view can miss two failure modes that matter for certification: (i) \emph{incoherent per-example geometry}, where loss landscapes are sharp only on data subsets but appear flat when averaged over the full dataset, and (ii) \emph{unstable flat minima}, where a model sits in an apparently flat, narrow pocket adjacent to sharp regions that a single perturbation may not reliably detect.
Both failure modes are linked to models with poor generalization~\cite{he2019asymmetric,andriushchenko_towards_2022}, overfitting, and backdoors~\cite{yuan2024activation,garg2024memorization}.

Our core insight is that a reliable metric should \emph{verify sharpness dynamically}: a model in a wide, flat basin will remain well-behaved under a \emph{sequence} of small, stochastic perturbations, while models with incoherent or unstable geometry will eventually be pushed into sharp regions. We call this property \emph{dynamic stability}. Convergent evidence from sharpness-aware optimization supports this view: the effectiveness of SAM depends critically on \emph{continuous, stochastic} perturbations computed on small mini-batches, while perturbations using the full dataset are substantially less effective~\cite{andriushchenko_towards_2022,li2024friendly,baek_why_2024}.

\smallskip\noindent\textbf{Measuring dynamic stability via \sharpness.}
Building on this insight, we design \sharpness, a generalization metric that measures the dynamic stability of a model. Given a trained model, we apply a short sequence of SAM-style, stochastic perturbations and record the resulting sharpness after each step. A model in a genuinely flat basin absorbs these perturbations and maintains stable sharpness values; a model in a narrow, unstable region will be pushed toward sharp neighbors, producing detectable fluctuations. As we show in Sections~\ref{sec:flat:corr} and~\ref{sec:exp}, this shift from static to dynamic measurement yields significantly more reliable predictions of generalization, even under training deviations.

\smallskip\noindent\textbf{Theoretical grounding.}
We provide theoretical analysis linking \sharpness to a formal notion of stability.
Under standard assumptions, we show that if the initial model $\vecw_0$ is near a stable minimum, then the expected squared sharpness at step $t$ under continuous perturbation is
uniformly bounded:
\[
\mathbb{E}[s_t^2]
\le
\overline{\kappa}_{\rho} \cdot C \cdot \LLL_D(\vecw_0)
\qquad\text{for all }t,
\]
where $C>0$ and $\overline{\kappa}_{\rho}$ are constants and $\LLL_D(\vecw_0)$ is the training loss of the initial model. Conversely, if the expected squared sharpness grows exponentially, then $\vecw_0$ resides in an unstable region, which makes \sharpness a good indicator of apparently flat but dynamically unstable regions.

We further analyze the robustness of \sharpness under training deviations. We formalize a class of deviations that induce low gradient coherence, a phenomenon linked to poor generalization, overfitting, and backdoors~\cite{chang2025a,Chatterjee2020Coherent,sankararaman2020impact}. For this class of deviations, we show that \sharpness can detect instability signals that static measures may miss.

We complement our analysis with empirical evaluations showing that \sharpness (i) correlates with generalization more strongly than static sharpness metrics, and (ii) more consistently detects low-quality models under training deviations like overfitting and backdoors.

\smallskip\noindent\textbf{Empirical validation and practical certification.}
We conduct extensive experiments validating \sharpness for practical model certification. We consider two scenarios: \emph{auditing}, where a trusted verifier has access to training data and computes \sharpness directly; and \emph{certification}, where an external verifier, without accessing training data, checks a zero-knowledge proof that a potentially adversarial provider's claimed \sharpness value was computed correctly.
We validate three properties that make \sharpness suitable for both scenarios:
\begin{itemize}[leftmargin=*]
    \item \emph{strong correlation with generalization}: In Section~\ref{sec:flat:corr}, we show that \sharpness correlates significantly more strongly with generalization than even the best-performing static metrics.
    \item \emph{reliability under training deviations}:
    Existing generalization metrics have been evaluated only in the context of honestly trained models.
    However, for many certification settings, it is important that \sharpness remain correlated with generalization even under training deviations.
    In Section~\ref{sec:exp}, we evaluate \sharpness on models with both intentional deviations (e.g., injected backdoors) and unintentional ones (e.g., overfitting). \Sharpness reliably flags these models even when both test accuracy and static sharpness fail, confirming that dynamic stability measurement is more sensitive to subtle forms of poor generalization.
    \item  \emph{efficiency}: In the auditing setting, \sharpness is $4\times$ faster to compute than test accuracy. In certification, its zero-knowledge proof is up to $80{,}000\times$ faster than proving training correctness step-by-step (88.87 minutes versus an estimated 118{,}671 hours for VGG-11 over 150 epochs).
\end{itemize}

\subsection{Related Work}
\noindent\textbf{ML model auditing and certification.}
We distinguish between the terms \textit{model auditing} and \textit{model certification} based on the verifier's ability to access the training data. Specifically, we define \textit{model auditing} as a setting where the verifier has access to the training data. Many existing techniques for checking specific trustworthy properties of ML models fall under this category: \textit{fairness} auditing ensures models do not discriminate against particular groups~\cite{faired}; \textit{differential privacy auditing}~\cite{annamalai2025hitchhikersguideefficientendtoend} empirically tests the theoretical guarantee that a model's output hides the presence of individual training examples; and \textit{robustness} auditing~\cite{robust1,robust2,robust3,robust4} aims to provide provable guarantees against adversarial examples.

Conversely, we define \textit{model certification} as a setting where the verifier has no access to the training data. Many works in privacy-preserving ML adopt this setting, where cryptographic techniques like zero-knowledge proofs allow the verifier to \textit{certify} trustworthy properties while preserving the privacy of the training data. Existing works include certification of properties such as fairness~\cite{zkfair,fairzk1,fairzk2,zkfair3,fair2,fair3,fair}, differential-privacy guarantees~\cite{zkdp}, and vicinity to an optimum~\cite{OV}. Beyond property-specific guarantees, zero-knowledge proofs of training (zkPoT) allow the verifier to certify that a model was obtained by faithfully running a prescribed training algorithm~\cite{sun2024zkdl,kaizen,garg2023experimenting}.

\smallskip\noindent\textbf{Generalization and sharpness.}
While existing methods can certify the \textit{integrity} of the training process or specific properties, they do not directly certify the most fundamental measure of model quality: its ability to generalize to unseen data.
Theoretically, a generalization certificate is a tight, computable upper bound on the population risk~\cite{perez2021tighter}. Many works have focused on deriving such generalization bounds for modern ML models~\cite{pac1,pac2,PACbook,dziugaite_search_2021,haddouche_pac-bayesian_2025,Foong_Bruinsma_Burt_Turner_2022,Neu_Dziugaite_Haghifam_Roy_2021,dlbound,Nagarajan_Kolter_2019,Dziugaite_Roy_2017}. However, existing generalization bounds are often vacuous or too loose when applied to practical neural networks, and thus cannot certify generalization in practice.

A parallel line of research has studied empirical proxies for generalization. Among these, the geometry of the loss landscape, specifically its \textit{sharpness}, has long been linked to generalization, with flatter minima generalizing better~\cite{keskar_large-batch_2017,jiang_fantastic_2019}. This connection has motivated optimization techniques, such as SAM, that explicitly seek flatter minima to improve generalization~\cite{sam,ASAM}. While sharpness has been studied extensively as a generalization proxy, to the best of our knowledge, it has not been explored as a practical certificate of generalization.

\section{Preliminaries}
\label{sec:prelim}
\subsection{Machine Learning}

In this paper, we consider a generalized framework for supervised learning where a model is parameterized by $\vecw \in \mathbb{R}^p$. Let $\data = \{(\vecx_i, y_i)\}_{i=1}^N$ denote the training dataset with inputs $\vecx_i \in \mathbb{R}^d$ and labels $y_i \in \mathbb{R}$. A loss function $\ell(\vecw, (\vecx, y))$ measures the quality of the model's prediction on a data point. The training loss is $\LLL_\data(\vecw) = \frac{1}{N}\sum_{i=1}^N \ell(\vecw, (\vecx_i, y_i))$, which we abbreviate as $\LLL(\vecw):=\LLL_D(\vecw)$ and $\ell_i(\vecw) := \ell(\vecw, (\vecx_i,y_i))$ when the dataset is clear from context. A training algorithm $\mathsf{Train}_{\LLL_\data} : \mathcal{R} \to \mathbb{R}^p$ takes a random seed $r \in \mathcal{R}$ and outputs a model $\vecw$ that approximately minimizes $\LLL_\data(\vecw)$. We write $g(\vecw) := \nabla \LLL_\data(\vecw)$ for the full-dataset gradient and $H(\vecw)$ for the Hessian, dropping $\vecw$ when clear from context.

In practice, optimization uses mini-batches $\xi \subset \data$ of size $B = |\xi|$ sampled uniformly at random. We write $\LLL_\xi(\vecw)$ for the loss on $\xi$, and define the stochastic gradient as $g_\xi(\vecw) := \nabla \LLL_\xi(\vecw)$. The gradient noise is $\zeta(\vecw, \xi) := g_\xi(\vecw) - \nabla \LLL_\data(\vecw)$.

\subsection{Generalization and Sharpness}
Following the standard statistical learning framework~\cite{PACbook,itbound}, we assume training examples $z_i:=(\vecx_i,y_i)$ are drawn i.i.d.\ from a fixed but unknown distribution $\mathcal{P}$ over an instance space $\mathcal{Z} = \mathcal{X} \times \mathcal{Y}$. The \emph{population risk} of a model $\vecw$ is 
$R(\vecw) = \mathbb{E}_{z \sim \mathcal{P}}[\ell(\vecw, z)],$
which cannot be computed directly since $\mathcal{P}$ is unknown. Instead, we estimate it using the \emph{empirical risk}:
$$\hat{R}(\vecw) = \frac{1}{N} \sum_{i=1}^N \ell(\vecw, z_i),$$ which can be computed directly on the training set.

The core challenge in learning is that a model with low empirical risk might not necessarily have a low population risk. The \emph{generalization gap}, $\operatorname{Gap}(\vecw) = R(\vecw) - \hat{R}(\vecw)$, quantifies how well training performance predicts test performance. A model generalizes well when this gap is small. In Section~\ref{sec:cert}, we discuss how to certify generalization given that the true gap is unobservable.

\smallskip\noindent\textbf{Generalization metrics and sharpness.}
A generalization metric, or \emph{complexity measure}~\cite{jiang_fantastic_2019,haddouche_pac-bayesian_2025,dziugaite_search_2021,viallard2024leveraging}, is a quantity that monotonically relates to generalization and is computable without access to a test set.
Among the many proposed generalization metrics, one of the most prominent and empirically successful is the \emph{sharpness} of the model~\cite{keskar_large-batch_2017,sam,ASAM}. In this paper, we focus on perturbation-based (worst-case) sharpness:
\begin{definition}[Worst-Case Sharpness]
\label{def:sharpness}
Given a loss function $\LLL_\data$ parameterized by the training data $\data$, a radius $\rho > 0$, and norm $\|\cdot\|_p$, the \textbf{worst-case sharpness} of parameters $\vecw$ is defined as:
$$\mathsf{S}_{\rho,p}(\vecw, \data) := \max_{\|\bm{\epsilon}\|_p \le \rho} \left( \LLL_\data(\vecw + \bm{\epsilon}) - \LLL_\data(\vecw) \right).$$
\end{definition}
This definition, however, is sensitive to parameter re-scaling~\cite{sharpcan,ASAM}, which can be exploited to make a sharp minimum appear artificially flat. This led to a magnitude-aware worst-case version invariant to such re-scaling:
\begin{definition}[Magnitude-Aware Worst-Case Sharpness~\cite{ASAM}]
\label{def:adapt_sharpness}
Given a loss function $\LLL_\data$ parameterized by the training data $\data$, a radius $\rho > 0$, and norm $\|\cdot\|_p$, denote $\mathbf{T}_{\vecw} := \textrm{diag}(|\vecw|)$. The \textbf{magnitude-aware worst-case sharpness} of parameters $\vecw$ is:
$$\mathsf{S}^{\mathsf{mag}}_{\rho,p}(\vecw, \data) := \max_{\|{\bf{ T}}^{-1}_{\vecw}\bm{\epsilon}\|_p \le \rho} \left( \LLL_\data(\vecw + \bm{\epsilon}) - \LLL_\data(\vecw) \right).$$

\end{definition}

\subsection{Zero-Knowledge Proofs}
A zero-knowledge proof (ZKP) of circuit satisfiability allows a prover holding a witness $w$ to prove to a verifier that $\CCC(w) = 1$ for some public circuit $\CCC$ without revealing any information about $w$.
Many existing ZKP protocols can be used for proving \sharpness, including zkSNARKs~\cite{EC:Groth16,EC:CHMMVW20,EPRINT:GabWilCio19,libsnark} and proofs based on vector oblivious linear evaluation (VOLE)~\cite{CCS:WYYXW22,C:BMRS21,CCS:DILO22,C:BBMS22,CCS:BBMRS21,CCS:YSWW21}.

\smallskip\noindent\textbf{Zero-knowledge proofs of training.}
Zero-knowledge proofs of training (zkPoTs)~\cite{sun2024zkdl,kaizen,garg2023experimenting,OV} allow a prover to prove that a model is trained correctly on a committed dataset without revealing any additional information about the model or the dataset. ZkPoTs can be used for distributed model training, privacy auditing, and proving ownership~\cite{10966041,kaizen,zkdp}.

\section{Certification of Model Generalization}
\label{sec:cert}

\subsection{Certification Threat Model}
We now motivate and formalize what it means to certify generalization. The goal of certification is to generate a (short) certificate that enables efficient verification of the certified property. Thus, there are two main steps in the certification process: computation of the certificate (proof) and verification of the certificate. The roles, the threat models, and the available inputs for the parties that execute these steps may differ in different applications. We distinguish between two main settings: \emph{model auditing} and \emph{model certification}.

In model auditing, the verifier is trusted by the model provider to have access to the training data. This includes settings where the model provider uses generalization certificates as a quality metric for candidate model checkpoints and to detect unintentional training issues such as overfitting or memorization. Since the model provider is the natural consumer of the generalization certificates, the threat model for the soundness of the certificates assumes that training has been done honestly.

In external auditing applications, the verifier is a third party, but the trust framework for the auditor provides access to training data as well. In this case, the model provider may or may not be trusted to execute the training honestly.

On the other hand, in model certification, the verifier and the model provider---acting now as a cryptographic \emph{prover}---are mutually distrustful. Thus, the verifier is not given access to the training data, and the prover is not trusted to follow any prescribed steps for the training. In this setting, the prover certifies a claim about the model relative to a committed training dataset $D$. We assume that $D$ is authenticated, so the prover cannot arbitrarily craft, replace, or modify the dataset used for certification. This assumption is consistent with existing proof-of-training literature~\cite{sun2024zkdl,kaizen,garg2023experimenting,OV} and is natural in practical certification settings, where the authenticity of $D$ can be enforced by external authentication, trusted data sources, or additional proofs. Attacks in which the provider can freely choose or forge the committed dataset are therefore outside our threat model.

Bearing in mind the above certification applications, we proceed to describe our approach to certifying generalization and our security goals. The first required step towards computing a generalization certificate is the ability to measure this property. Next, we discuss the challenges to quantifying generalization.

\subsection{Quantifying Generalization}
\label{sec:cert:quant}

As discussed in Section~\ref{sec:intro}, for ``certifying generalization'' to have practical meaning, it cannot mean proving the exact generalization gap or a loose theoretical bound. Instead, we seek to certify an empirically motivated \textit{generalization metric} that is computable and has direct, actionable implications in practice. We now formalize what it means to certify generalization via such metrics.

\noindent\textbf{Metric-based certification.}
In our solution, we adopt a metric-based approach in which the generalization certificate for the model is based on a generalization metric computed on the final model.
Specifically, a \emph{generalization metric} is an efficiently computable (randomized) function $\gamma_\LLL$ mapping $(\vecw, D, r)$ to a value $\gamma_\LLL(\vecw, D, r) \in \mathbb{R}_{\ge 0}$,
where $D$ and $\vecw$ denote the training set and trained model held by the model provider (prover), $\LLL$ is a public loss function, and $r$ is a public random seed. Smaller values of $\gamma_\LLL(\vecw, D, r)$ are interpreted as stronger evidence of generalization.

Therefore, by \textit{certificate of generalization} we mean a certificate (proof) of a predicate on this metric:  $$\mathsf{Pred}_\tau(\vecw, D,r) \equiv [\gamma_\LLL(\vecw, D,r) \le \tau],$$ where the threshold $\tau$ and seed $r$ are public. When instantiated with a ZKP with knowledge soundness, this means that any computationally bounded prover who convinces the verifier must know a model $\vecw$ and dataset $D$ satisfying $\mathsf{Pred}_\tau(\vecw,D,r)$. Thus, the prover cannot falsely claim a lower metric value for the model $\vecw$ and its committed dataset $D$.
Importantly, the role of certification here is to verify the value of $\gamma_\LLL(\vecw, D,r)$ and not to prove the link between $\gamma_\LLL$ and the true generalization gap.

\subsection{Desiderata of Generalization Metrics in Certification}
\label{sec:flat:desider}

Having defined metric-based certification, we now ask what properties a generalization metric must satisfy to serve as the basis of a certificate.

\smallskip\noindent\textbf{Existing evaluation method.} Because generalization does not have an observable ``ground truth'', evaluating how well a metric predicts generalization must be done through some proxy. Most empirical work evaluates a candidate metric by training a large grid of models, computing the metric on each model, and measuring its correlation with an empirical proxy for generalization—most commonly the \emph{empirical generalization gap}  measured on a holdout test set (e.g.,~\cite{jiang_fantastic_2019}). Metrics with stronger and more consistent correlation are viewed as more predictive of generalization.

\smallskip\noindent\textbf{Evaluating metrics for certification.}
However, existing metrics are primarily proposed as \emph{explanatory} tools for understanding generalization in settings where models are trained honestly and evaluation data is benign. In certification, the training procedure may deviate from the nominal one, either accidentally (e.g., label noise, shortcut learning, or distributional artifacts) or intentionally (e.g., data poisoning or backdoors). Thus, a metric used for certification requires stronger evidence than correlation on honestly trained models alone.

Specifically, we propose the following desiderata that a metric should satisfy for the resulting certificate to be reliable.
\begin{enumerate}[leftmargin=*]

\item \textbf{Correlated with generalization.}\label{item:desiderata-predictive}  As an empirical generalization metric, it must correlate strongly with the empirical generalization gap across a diverse set of models. This is the established baseline for any generalization metric.
\item \textbf{Predictive despite training deviations.} \label{item:desiderata-robust}
The metric should continue to track generalization even when the training process deviates from the nominal procedure (intentionally or accidentally).
Concretely, this means it should be able to identify training deviations that preserve low training loss but degrade generalization, such as overfitting, memorization, spurious shortcuts, or adversarial behaviors that are only triggered under specific conditions.

\item \textbf{Efficiently computable.} \label{item:desiderata-efficient} As a practical certification tool, it must be computationally efficient, requiring at most a handful of mini-batch operations and avoiding full-dataset or Hessian-level costs.

\end{enumerate}
\smallskip\noindent\textbf{On robustness to training deviations.}
While a generalization metric used for certification should ideally provide guarantees against any training deviations, we note that formalizing such guarantees faces fundamental challenges. First, the relationship between training deviations and generalization is not well defined. Generalization, in the strict learning-theoretic sense, is defined with respect to a fixed data distribution $\mathcal{P}$. For example, a backdoor model might behave maliciously only on attacker-chosen trigger inputs that lie outside the support of $\mathcal{P}$; such behavior may not increase the generalization gap as classically defined. Second, ruling out \emph{all} strategies that fool a metric would require quantifying over unconstrained adversarial objectives. Some objectives may be statistically independent of the training set and therefore cannot be detected from the model and training data alone; without additional assumptions, universal detection can be impossible in a strong sense.

Given these challenges, we instead take a principled approach to validate our metric. In Section~\ref{sec:flat:deviation}, we provide analysis of a class of deviations characterized by \emph{low gradient coherence}---a property empirically linked to poor generalization and known deviations---and show that such deviations can remain invisible to existing metrics yet can be detected by \sharpness. In Section~\ref{sec:exp}, we complement our analysis with extensive experimental evidence that \sharpness reliably flags models with known deviations even when test accuracy and static sharpness do not.

\section{Defining and Validating \SHARPNESS}
\label{sec:sharpness}
Section~\ref{sec:cert} formalized certifying generalization and identified three desiderata for generalization metrics: strong correlation, reliability under training deviations, and computational efficiency. In this section, we develop such a metric based on the dynamic stability of the loss landscape. We begin by reviewing the stochastic behavior of the SAM optimizer and deriving our metric from its dynamics (Sections~\ref{sec:flat:sam_dynamics} and~\ref{sec:flat:dynamic}), then provide theoretical grounding (Sections~\ref{sec:flat:math} and~\ref{sec:flat:deviation}), and present empirical validation (Section~\ref{sec:flat:corr}).

\subsection{From SAM to Dynamic Stability}
\label{sec:flat:sam_dynamics}
The starting point of our generalization metric is Sharpness-Aware Minimization (SAM)~\cite{sam}.
SAM seeks flat minima by minimizing both the loss and its \textit{worst-case sharpness} (Def.~\ref{def:sharpness}), yielding the min-max objective:
\begin{equation}
\min_{\vecw} \max_{\|\bm{\epsilon}\|_p \le \rho} \LLL_D(\vecw + \bm{\epsilon}).
\label{eq:sam_objective}
\end{equation}
Solving the inner maximization exactly is intractable, so SAM approximates it via a first-order Taylor expansion: $\LLL_D(\vecw + \bm{\epsilon}) \approx \LLL_D(\vecw) + \bm{\epsilon}^\top \nabla \LLL_D(\vecw)$. The inner problem then reduces to $\max_{\|\bm{\epsilon}\|_p \le \rho} \bm{\epsilon}^\top \nabla \LLL_D(\vecw)$, which is a dual norm problem with the closed-form solution (for \(1<p<\infty\) and \(\nabla\LLL_D(\vecw)\neq 0\)):
\begin{equation}
\hat{\bm\epsilon}(\vecw)=\rho\,\frac{\text{sign}(\nabla\LLL_D(\vecw))\,|\nabla\LLL_D(\vecw)|^{q-1}}{\|\nabla\LLL_D(\vecw)\|_q^{\,q-1}},
\label{eq:sam_perturbation}
\end{equation}
where $q$ satisfies $\frac{1}{p} + \frac{1}{q} = 1$. Substituting this back yields:
\begin{definition}[Approximated Sharpness~\cite{sam}]
\label{def:approx_sharpness}
Let $\LLL_D$ be the loss, and let $\nabla \LLL_D(\vecw)$ be its gradient. Given a radius $\rho > 0$ and a pair of dual $p, q$ norms (where $\frac{1}{p} + \frac{1}{q} = 1$), the approximated sharpness or \textbf{SAM sharpness} of model $\vecw$ is:
\[
\mathsf{S}_{\rho,p}^{\mathsf{sam}}(\vecw, D) := \nabla \LLL_D(\vecw)^\top \hat{\bm{\epsilon}}(\vecw) = \rho \|\nabla \LLL_D(\vecw)\|_q.
\]
\end{definition}
ASAM~\cite{ASAM} proposes using the magnitude-aware worst-case sharpness (Def.~\ref{def:adapt_sharpness}) as a more robust objective; this yields:
\begin{definition}[Approximated Magnitude-Aware Sharpness~\cite{ASAM}]
\label{def:approx_adapt_sharpness}
Let $\mathbf{T}_{\vecw}$ denote $\textrm{diag}(|\vecw|)$. Given a radius $\rho > 0$ and a pair of dual $p, q$ norms (where $\frac{1}{p} + \frac{1}{q} = 1$), the {magnitude-aware approximated sharpness} or \textbf{ASAM sharpness} of $\vecw$ is:
\[
\mathsf{S}_{\rho,p}^{\mathsf{asam}}(\vecw,D):=\max_{\|{\bf{ T}}^{-1}_{\vecw}\bm{\epsilon}\|_p \le \rho} \left( \nabla \LLL_D(\vecw)^\top \bm{\epsilon} \right) =
\rho \|{\bf{ T}}_{\vecw}\nabla \LLL_D(\vecw)\|_q.
\]
\end{definition}

SAM effectively makes two approximations: (i) replacing the worst-case perturbation with the gradient direction, and (ii) replacing full-batch loss with mini-batch loss.
Recent work reveals that these approximations are not merely computational shortcuts but are \emph{essential} for improving generalization: using full-batch sharpness significantly degrades performance~\cite{sam,baek_why_2024}. This suggests that SAM succeeds not merely because it finds points of low sharpness, but because it finds minima that remain stable under \emph{continuous}, \emph{stochastic} perturbation---two properties that existing static measures lack:
\begin{itemize}[leftmargin=*]
\item \textbf{Continuous perturbation.} Static sharpness measures evaluate the loss increase from a single, bounded perturbation, which can be misleading when a minimum is locally flat but adjacent to sharp regions. A sequence of perturbations can incrementally push parameters out of such narrow pockets, revealing the adjacent sharpness through detectable increases in sharpness. This is confirmed by~\cite{latesam}: applying SAM to an SGD-trained model pushes it from the SGD minimum to a much flatter SAM minimum.

\item \textbf{Stochastic perturbation.}
Static measures compute sharpness on the full training set and can potentially mask high sharpness on specific subsets. A model that appears flat on the full dataset but exhibits high sharpness on particular data points would likely lead to poor generalization.
\end{itemize}

These findings motivate our central proposal: we define a generalization metric that measures sharpness not as a static property at a single point, but as a measure of \emph{dynamic stability}---how consistently flat a model remains under continuous, sharpness-aware perturbation. We detail our approach next.

\subsection{Our Approach: Dynamic Stability}
\label{sec:flat:dynamic}

\RestyleAlgo{boxed}
\begin{algorithm}[t]
\DontPrintSemicolon
\KwIn{Initial model parameters $\vecw_0$, dataset $D$, public random seed $r$, configuration
$\mathcal{C}=(T,B,\Phi,\mathsf{S},\mathsf{fluc},\LLL,\Theta)$.}
\KwOut{\SHARPNESS measure $\mathcal{F}$.}

$\vecw \leftarrow \vecw_0$\;
$\mathcal{H} \leftarrow [\,]$\;

\For{$t \leftarrow 0$ \KwTo $T-1$}{
    Sample a mini-batch $\xi_t \subset D$ of size $B$ via $(r,t)$\;
    $\LLL_{\xi_t} \leftarrow \LLL_{\xi_t}(\vecw)$\;
    $s_t \leftarrow \mathsf{S}(\vecw,\xi_t,\Theta)$\;
    Append $(\vecw,\xi_t,s_t,\LLL_{\xi_t})$ to $\mathcal{H}$\;
    $\vecw \leftarrow \Phi(\vecw,\xi_t,\Theta)$\;
}

$\mathcal{F} \leftarrow \mathsf{fluc}(\mathcal{H})$\;
\Return $\mathcal{F}$\;
\caption{Computing \sharpness.}
\label{fig:alg:flat}
\end{algorithm}

We define our generalization metric, \sharpness, as a measure of a model's geometric stability under a series of sharpness-aware perturbations.

\begin{definition}[\SHARPNESS]\label{def:dflat}
Given a model $\vecw_0$, a dataset $D$, a public random seed $r$, batch size $B$, number of steps $T$, a loss function $\LLL(\cdot)$, a sharpness function $\mathsf{S}(\cdot)$, a SAM update operator $\Phi_{\Theta}(\cdot)$ with hyperparameters $\Theta$, and a fluctuation statistic $\mathsf{fluc}(\cdot)$, for $t\in\{0, \dots,T-1\}$, we first generate a probe history $\mathcal{H} = \{h_0, \dots, h_{T-1}\}$ by iteratively computing:
$h_t := (\vecw_t, \xi_t, s_t, \LLL_{\xi_t})$
where $\xi_t$ is a mini-batch of size $B$ drawn uniformly from $D$ using seed $r$, $s_t = \mathsf{S}(\vecw_t, \xi_t,\Theta)$, $\LLL_{\xi_t} = \LLL_{\xi_t}(\vecw_t)$, and $\vecw_{t+1} := \Phi_{\Theta}(\vecw_t, \xi_t)$.
The \textbf{\sharpness} of model $\vecw_0$ parameterized by  $\mathcal{C} = (T, B, \Phi, \mathsf{S}, \mathsf{fluc}, \LLL, \Theta)$ is defined as:
$$\mathsf{S}^{\mathsf{d}}_{\mathcal{C}}(\vecw_0, D, r) := \mathsf{fluc}(\mathcal{H}).$$
\end{definition}

Algorithm~\ref{fig:alg:flat} shows the procedure for computing \sharpness. Given a trained model $\vecw_0$, we apply $T$ steps of SAM-style updates using freshly sampled mini-batches. At each step $t$, we record a per-step sharpness $s_t = \mathsf{S}_{\rho,p}^{\mathsf{sam}}(\vecw_t,\xi_t)$ (e.g., SAM sharpness or ASAM sharpness) with mini-batch $\xi_t$. This probing process yields a time series of SAM sharpness values $\{s_0, s_1, \dots, s_{T-1}\}$. We define the \sharpness metric using a statistic that quantifies the stability of this sequence.

\smallskip\noindent\textbf{Choice of fluctuation statistic.}
Here $\mathsf{fluc}$ is a function that outputs a final non-negative number quantifying the amount of fluctuation in $\mathcal{H}$. In practice, we found that tracking the raw sharpness values $\{s_t\}$ alone can be sensitive to the overall loss scale across models and datasets. We therefore use an \emph{empirical} normalization strategy via $r_t:=s^2_t/\LLL_{\xi_t}$ and set $\mathsf{fluc}(\mathcal{H}):=\mathrm{Std}(\{\log(r_t+\varepsilon)\}_{t=0}^{T-1})$ for a small $\varepsilon>0$. Empirically, this ratio-based statistic is slightly more correlated with generalization than using $\{s_t\}$ alone, so we adopt it by default.

The key semantic difference from static sharpness is that \sharpness tests how a model \emph{responds} to SAM dynamics rather than measuring geometry at a single point.
If $\vecw_0$ resides in a truly flat basin—the type SAM converges to—then applying SAM perturbations will produce stable sharpness values. Conversely, if $\vecw_0$ resides in a sharp or unstable region, SAM will push the model toward neighboring sharp areas, producing an unstable series. Thus, a low \sharpness value suggests that the model remains geometrically
stable under the SAM probing dynamics, while a high value indicates instability along this trajectory. We validate empirically that this dynamic-stability signal correlates with generalization and provide a theoretical analysis of this connection in the next section.

\subsection{A Theoretical Foundation for \SHARPNESS}
\label{sec:flat:math}

In this section, we provide a theoretical analysis for our \sharpness metric.
We proceed by defining a notion of stability based on recent results on SAM dynamics~\cite{latesam}, then deriving bounds that connect our sharpness measurements to this stability.

\smallskip\noindent\textbf{Formalizing stability.}
Recall that the goal of \sharpness is to verify whether a trained model exhibits stability properties that are predictive of good generalization.
We formalize this stability property using the concept of {SAM-stability} from recent optimization literature, as recent work suggests that SAM's dynamics capture a generalization-relevant bias stronger than static sharpness measures: among multiple minima of equal static sharpness, SAM exhibits an implicit bias favoring those that generalize better~\cite{latesam,andriushchenko2023sharpness,wen_sharpness_2023,springer2024sharpnessaware}.

Concretely, we consider the model we are certifying to have a small training loss and be near some global minimum $\vecw^*$. In this regime, the model can be approximated by the local linearization of its prediction function $f$:
$$f_{\text{lin}}(\vecx; \vecw) = f(\vecx; \vecw^*) + \langle \nabla_{\vecw} f(\vecx; \vecw^*), \vecw - \vecw^* \rangle.$$
The corresponding linearized loss is:
$$ \LLL_{\text{lin}}(\vecw)
  =
  \frac{1}{N} \sum_{i=1}^N
  \ell\bigl(f_{\text{lin}}(\vecx_i; \vecw), y_i\bigr).$$
This linearization allows us to analyze the local geometry of the loss landscape near convergence. Throughout our analysis, we work in this linearized model and write $\LLL(\vecw) := \LLL_{\text{lin}}(\vecw)$. We define stability in terms of SAM's behavior on $\LLL$:
\begin{definition}[SAM-Stability]\label{def:sam_stability}
A minimum $\vecw^*$ is (linearly) \emph{SAM-stable} if there exists $C > 0$ such that for any starting point $\vecw_0$ near $\vecw^*$ where the linearization is valid, applying the SAM optimizer to the loss $\LLL$ yields a trajectory $\{\vecw_t\}$ satisfying $\mathbb{E}[\LLL(\vecw_t)] \le C \cdot \mathbb{E}[\LLL(\vecw_0)]$ for all $t \ge 0$. A minimum that is not SAM-stable is called \emph{SAM-unstable}. 
\end{definition}

Intuitively, a SAM-stable minimum is one where SAM's perturbation-based dynamics remain bounded---the optimizer cannot escape, and the loss stays controlled.
Our central claim is that \sharpness provides an efficient, empirical test for this stability. We show that bounded SAM loss implies bounded expected sharpness, while exponential sharpness growth certifies instability.
Thus, a stable basin can yield controlled sharpness along the probing trajectory, and exponential growth rules out a stable basin.

\smallskip\noindent\textbf{Linking \sharpness to stability.}
We establish this claim in two steps. First, we derive a ``sandwich'' bound showing that the sharpness $s_t$ is tightly coupled to the loss $\LLL(\vecw_t)$ at each step $t$. Second, we use this bound to characterize how sharpness evolves at stable versus unstable minima.

We first adopt the following assumptions from~\cite{latesam} and establish that per-step SAM sharpness measurements are coupled to the loss:
\begin{assumption}[Smoothness and Polyak-Łojasiewicz Condition \cite{latesam}]
\label{as:PL}
There exist $\mu>0, L_s>0$ such that $\|H(\vecw)\|_2\leq L_s$ and $\|g(\vecw)\|_2^2\geq 2\mu\,\LLL(\vecw)$ for all $\vecw$.
\end{assumption}

\begin{assumption}[Bounded Gradient Noise~\cite{latesam}]
\label{as:noise}
The gradient noise variance $\zeta(\vecw, \xi) = g_\xi(\vecw) - g(\vecw)$ is bounded by the full-dataset loss. That is there exists a constant $\sigma > 0$ such that for all $\vecw$, $\mathbb{E}[\|\zeta(\vecw, \xi)\|_2^2] \le \sigma^2 \LLL(\vecw).$
\end{assumption}

Assumption~\ref{as:PL} is standard and frequently used in non-convex optimization. Assumption~\ref{as:noise} states that mini-batch noise is controlled by the loss value, a phenomenon that has been empirically observed~\cite{mori2022power,feng2021inverse,wojtowytsch2024stochastic}.

\begin{lemma}[Linear Sandwich Bound]
\label{thm:linear-sandwich}
Let $\{\vecw_t\}_{t\ge 0}$ be a sequence of models where Assumptions~\ref{as:PL} and~\ref{as:noise} hold. Let $s_t$ denote the SAM sharpness (Def.~\ref{def:approx_sharpness}) of model $\vecw_t$ with $p=q=2$ and radius $\rho$, evaluated over a mini-batch of size $B$ sampled uniformly without replacement from $D$. Then for all $t$:
$$\underline{\kappa}_\rho \cdot \LLL(\vecw_t) \le \mathbb{E}\bigl[s_t^2 \big| \vecw_t\bigr] \le \overline{\kappa}_\rho \cdot \LLL(\vecw_t),$$
where $\underline{\kappa}_\rho = 2\rho^2\mu$ and $\overline{\kappa}_\rho = \rho^2\left(2L_s + \frac{\sigma^2}{B}\right)$.
\end{lemma}
\begin{proof}
See Appendix~\ref{app:proof:sandwich}.
\end{proof}

This lemma shows that the expected squared sharpness value is bounded between constant multiples of the loss, with constants depending on gradient noise and local curvature. We now present our main theorem showing the link between \sharpness and SAM-stability.
\begin{theorem}[Sharpness Dynamics and SAM-Stability]
\label{thm:sharpness-dynamics}
Let $\{\vecw_t\}_{t\ge 0}$ be the trajectory generated by applying SAM to
$\LLL$, starting from $\vecw_0$ near a global minimum $\vecw^*$. Suppose
Assumptions~\ref{as:PL} and~\ref{as:noise} hold, and let $s_t$ denote the
mini-batch SAM sharpness at step $t$. Then the following statements hold:

\noindent\textbf{(a) Stability implies bounded sharpness.}
If $\vecw^*$ is SAM-stable in the sense of Def.~\ref{def:sam_stability}
with constant $C>0$, then
\[
\mathbb{E}[s_t^2]
\le
\overline{\kappa}_\rho C\LLL(\vecw_0)
\qquad\text{for all }t\ge 0.
\]

\smallskip\noindent\textbf{(b) Exponential sharpness growth certifies instability.}
Assume $\LLL(\vecw_0)>0$. If there exist $K>0$, $\alpha>1$, and $t_0\ge 0$ such that for all $t\ge t_0$, $\mathbb{E}[s_t^2]\ge K \alpha^t\LLL(\vecw_0),$ then $\vecw^*$ is SAM-unstable.
\end{theorem}

\begin{proof}
    See Appendix~\ref{app:proof:dynamic}.
\end{proof}

\smallskip\noindent\textbf{Interpretation.}
Theorem~\ref{thm:sharpness-dynamics} provides theoretical support for our claim that \sharpness is directly linked to SAM optimizer dynamics.
If the iterate lies in a SAM-stable basin, then the expected sharpness is uniformly bounded (Part (a)). Conversely, exponential growth of the expected sharpness is a sufficient certificate that the iterate is not SAM-stable (Part (b)). While the theorem is one-way, empirically we find that poor generalization and training deviations consistently lead to large \sharpness scores (Section~\ref{sec:exp}).

\smallskip\noindent\textbf{Remark (Extension to ASAM).}
For clarity, our analysis is stated for SAM sharpness (Def.~\ref{def:approx_sharpness}).
The same argument extends to ASAM (Def.~\ref{def:approx_adapt_sharpness}), where the perturbation is constrained by
$\|({\bf T}_{\vecw})^{-1}\bm{\epsilon}\|_p\le \rho$ for a scaling operator $\mathbf{T}_{\vecw} := \textrm{diag}(|\vecw|)$.
In particular, ASAM uses ${\bf T}'_{\vecw}={\bf T}_{\vecw}+\epsilon_0 I$ with $\epsilon_0>0$~\cite{ASAM}, so the
per-coordinate scaling is bounded away from $0$.
Assume each weight magnitude is upper bounded, i.e., $\max_i(|(\vecw_t)_i|+\epsilon_0)\le M$ for all $t\in\{0,\dots,T-1\}$.
Then Lemma~\ref{thm:linear-sandwich} holds for ASAM with
$\underline{\kappa}_\rho = 2\epsilon_0^2\rho^2\mu$ and
$\overline{\kappa}_\rho = M^2\rho^2\!\left(2L_s + \frac{\sigma^2}{B}\right)$.

\subsection{Detecting Training Deviations}
\label{sec:flat:deviation}
We now analyze why \sharpness can detect poor generalization under training deviations that full-batch sharpness measurements miss. Specifically, we formalize a class of training deviations with low \textit{gradient coherence}, which has been empirically linked to poor generalization behavior and observed in known deviations; we then show that these deviations can be invisible to the full-batch sharpness metric but detectable by \sharpness.

\smallskip\noindent\textbf{Gradient coherence measure.}
Recent work~\cite{chang2025a} links flat minima and generalizable solutions to the notion of \textit{data coherence}: how consistently different training examples ``agree'' on the local geometry. Given per-example Hessians $H_i := \nabla^2 \ell_i(\vecw)$, the authors define a coherence score over the matrix $\mathbf{S} \in \mathbb{R}^{N \times N}$ with entries $\mathbf{S}_{ij} := \sqrt{\mathrm{Tr}(H_i H_j)}$. Intuitively, $\mathbf{S}_{ij}$ is large when examples $i$ and $j$ have {aligned curvature}.

Data coherence has also been studied in terms of gradient alignment~\cite{Chatterjee2020Coherent,sankararaman2020impact}, which can be captured using the gradient Gram matrix $G$ with entries $G_{ij} := \langle g_i(\vecw), g_j(\vecw)\rangle$ where $g_i(\vecw) := \nabla \ell_i(\vecw)$ is the per-example gradient. Intuitively, large off-diagonal entries in $G$ indicate that two examples have aligned gradients. We can define a coherence score on $G$ by measuring how strongly per-example gradients agree on the average direction:
\begin{definition}[Gradient Coherence]
\label{def:grad-coh}
Let $g_i(\vecw) := \nabla \ell_i(\vecw)$ and $g(\vecw) := \frac{1}{N}\sum_{i=1}^N g_i(\vecw)$. The \emph{gradient coherence} of model $\vecw$ on dataset $D$ is:
\[
c_g(\vecw) :=\frac{\frac{1}{N^2}\sum_{ij}G_{ij}}{\frac{1}{N}\mathrm{Tr}(G)} =  \frac{\|g(\vecw)\|_2^2}{\frac{1}{N}\sum_{i=1}^N \|g_i(\vecw)\|_2^2} \in [0, 1].
\]
\end{definition}
\noindent Intuitively, high $c_g$ means per-example gradients are largely aligned, while low $c_g$ means substantial cancellation happens across gradients.

\smallskip\noindent\textbf{Training deviations induce low coherence.}
Empirically, training deviations often manifest as poorly aligned (incoherent) gradients/curvature across examples. In backdoor attacks, Yuan et al.~\cite{yuan2024activation} observe a substantially more dispersed distribution of activation gradients within the backdoor target class, indicating reduced within-class alignment of per-example geometry. Separately, Garg et al.~\cite{garg2024memorization} show that memorized examples correlate with high per-example local curvature, suggesting strong per-example heterogeneity rather than shared curvature structure.
Motivated by these observations, we formalize the class of training deviations using gradient coherence:
\begin{assumption}[Low-Coherence Deviations]
\label{as:low-coh}
A training deviation produces a model $\vecw$ that achieves low training loss on a clean dataset $D$ of size $N$ with small gradient $\|g(\vecw)\|_2 \leq \varepsilon$ and low gradient coherence $c_g(\vecw) \leq c_0 \ll 1$.
\end{assumption}
Our assumption says deviations produce models with low training loss and small gradient, yet these models have hidden poor generalization behavior manifested through low gradient coherence.

\smallskip\noindent\textbf{\Sharpness detects low gradient coherence.}
We now demonstrate that \sharpness can detect low-coherence deviations in cases where static full-batch sharpness cannot. Consider a single step of \sharpness measuring the mini-batch SAM sharpness $\mathsf{S}_\xi=\rho\|g_\xi(\vecw)\|_2$ and a static, full-batch SAM sharpness metric $\mathsf{S}=\rho\|g(\vecw)\|_2$. To compare the scale of the random mini-batch signal with the deterministic full-batch signal, define the root-mean-square (RMS) mini-batch sharpness
$\mathsf{S}_{\mathrm{rms}}(\vecw)
:=
\sqrt{\mathbb{E}_\xi[\mathsf{S}_\xi^2]}.$
The next theorem shows that under low-coherence deviations, the full-batch sharpness $\mathsf{S}$ can be small while the RMS mini-batch sharpness remains large. Moreover, the squared RMS gap
$\mathsf{S}_{\mathrm{rms}}^2-\mathsf{S}^2$ grows as the batch size $B$ decreases.

\begin{theorem}[Detection Gap]
\label{thm:detection-gap}
Let $\vecw$ be a model produced by a training deviation satisfying Assumption~\ref{as:low-coh}. Let $\xi$ be a mini-batch of size $B$ sampled uniformly from $[N]$ without replacement and $g_i := \nabla \ell_i(\vecw)$. Then the squared RMS gap between mini-batch sharpness and full-batch sharpness is:
\begin{align}
\label{eq:gap-diff}
\mathsf{S}_{\mathrm{rms}}^2 -\mathsf{S}^2
\;=\;
\bigl(\;\rho^2\cdot \frac{N-B}{B(N-1)}\cdot\frac1N\sum_{i=1}^N\|g_i\|_2^2\;\bigr)\;\cdot\;\bigl(1-c_g(\vecw)\bigr).
\end{align}
In particular, when $B\ll N$ and $c_g(\vecw)\ll 1$, this gap is dominated by
\begin{align}
\label{eq:dom}
\mathsf{S}_{\mathrm{rms}}^2-\mathsf{S}^2
\;\approx\;
\rho^2\cdot \frac{1}{B}\cdot \frac1N\sum_{i=1}^N\|g_i\|_2^2.
\end{align}
Moreover, if $\|g(\vecw)\|_2\le \varepsilon$ and $\varepsilon^2 \ll \frac{1}{B}\cdot \frac1N\sum_{i=1}^N\|g_i\|_2^2$, then the full-batch sharpness is negligible relative to the RMS mini-batch sharpness, i.e., $\mathsf{S}\ll \mathsf{S}_{\mathrm{rms}}.$
\end{theorem}
\noindent The proof appears in Appendix~\ref{app:detection-gap-proof}.

\smallskip\noindent\textbf{Implications for \sharpness.}
Intuitively, Theorem~\ref{thm:detection-gap} shows that full-batch sharpness only sees the \emph{average} gradient, while mini-batch sharpness is influenced by the \emph{typical per-example} gradient magnitude. Thus, if per-example gradients disagree, they can cancel in the full-batch average, making $\|g(\vecw)\|_2$ small even when the per-example gradients themselves have large norm. In other words, full-batch sharpness can average away high per-example gradient signals, making low-coherence deviations difficult to detect.

In contrast, these deviations can be detected by mini-batch sharpness. Eq.~\ref{eq:dom} shows that under low-coherence deviations, the squared RMS gap is controlled by the average per-example gradient scale $\frac1N\sum_i\|g_i\|_2^2$ and grows as $\approx 1/B$. Equivalently, the RMS mini-batch sharpness scale grows as $\approx 1/\sqrt{B}$. Thus, even when $\mathsf{S}$ is small, the mini-batch second-moment signal $\mathbb{E}_\xi[\mathsf{S}_\xi^2]$ can remain large, and using a smaller batch size makes this gap easier to detect.

Finally, \sharpness can amplify this one-step RMS gap by aggregating mini-batch sharpness over $T$ steps. Results from~\cite{chang2025a} suggest that low-coherence solutions are less stable under SAM-like dynamics; since \sharpness includes a SAM-style probing trajectory, low-coherence models are more likely to exhibit larger fluctuation or escape behavior than coherent ones, making them more detectable.

\subsection{Empirical Analysis of \SHARPNESS}
\label{sec:flat:corr}

Our theoretical analysis provides evidence that by measuring \textit{dynamic stability} rather than a \textit{static snapshot}, \sharpness could serve as a more robust and reliable predictor of generalization. We now validate \sharpness against Desideratum~\ref{item:desiderata-predictive} by evaluating its empirical correlation with the generalization gap.

\smallskip\noindent\textbf{Setup.}
We mimic the experimental setup of the large-scale experiment by Jiang et al.~\cite{jiang_fantastic_2019} and train a large grid of models on CIFAR-10 and CIFAR-100~\cite{cifar10} with different optimizers, model architectures, and hyperparameters. For CIFAR-10, we use VGG with different sizes: {VGG-13-BN, VGG-16-BN, and VGG-19-BN~\cite{simonyan2015deepconvolutionalnetworkslargescale}}; for CIFAR-100, we use WideResNet28-10~\cite{zagoruyko2016wide}. We train these models using the optimizers {SGD, Adam, SAM~\cite{sam}, ASAM~\cite{ASAM}}. For Adam, we choose learning rates $\{0.001, 0.0005, 0.00032, 0.0001\}$, while for the other optimizers we use $\{0.1, 0.032, 0.01, 0.05\}$. We sweep weight decay in $\{0.0001, 0.00005\}$, batch size in $\{32, 64, 128\}$, and dropout rate in $\{0.0, 0.25, 0.5\}$. All models are trained to a 0.01 cross-entropy loss (estimated on 100 mini-batches), and we discard any that fail to converge within 200 epochs.
For each converged model, we estimate the generalization gap using the difference between the test and training losses $\operatorname{Gap}(\vecw) \approx \LLL_{D_{test}}(\vecw) - \LLL_{D_{train}}(\vecw)$ and compute its correlation with \sharpness, ASAM sharpness (\textbf{ASAM}), and the best-performing magnitude-aware worst-case sharpness (\textbf{Worst-Case}) reported by Jiang et al.~\cite{jiang_fantastic_2019}.
Following~\cite{jiang_fantastic_2019}, we report Spearman's $\rho$, Kendall's $\tau$, and the granulated Kendall's coefficient $\Psi$.

\begin{table}[!t]
\centering
\small
\begin{tabular}{l @{\hskip1pt} ccc}
\toprule
\multirow{2}{*}{\textbf{Sharpness Metric}} & \multicolumn{3}{c}{\textbf{Correlation with Generalization Gap}} \\
\cmidrule(lr){2-4}

& \textbf{Spearman $\rho$} & \textbf{Kendall $\tau$} & \textbf{Kendall $\Psi$}  \\
\midrule
\textbf{Ours} $(B=16, T=5)$ & \textbf{0.902} & \textbf{0.702}  & \textbf{0.694} \\
\textbf{Ours} $(B=8, T=5)$  & \textbf{0.811} & \textbf{0.712}  & \textbf{0.676}\\
\cmidrule(lr){1-4}
\textbf{ASAM} $(B=256)$          & 0.581 & 0.264  & 0.194\\
\textbf{ASAM} $(B=|D|)$     & 0.639  & 0.433          & 0.411\\
\cmidrule(lr){1-4}
\textbf{Worst-Case} ($B=|D|$)                & 0.601  & 0.577 & 0.552\\
\bottomrule
\end{tabular}
\caption{Correlation between different sharpness measures and generalization gap on a grid of $1{,}152$ models with different architectures, optimizers, and hyperparameters. $B$ denotes the batch size used for the sharpness computation. Our \sharpness measure has a significantly stronger correlation than the static magnitude-aware worst-case sharpness (\textbf{Worst-Case}) and ASAM sharpness (\textbf{ASAM}) baselines. 
}

\label{tab:corr}
\end{table}
\smallskip\noindent\textbf{Results and discussion.}
Table~\ref{tab:corr} shows all three correlation coefficients for our \sharpness measure (\textbf{Ours}) compared against several existing sharpness measures. Our \sharpness measure shows a significantly stronger correlation with the generalization gap than all other sharpness measures.
These findings provide empirical support for our theoretical analysis and validate our core insight: reframing sharpness from a static geometric property to a measure of dynamic stability yields a metric that is more predictive of generalization than existing approaches.

This section confirms that \sharpness meets Desideratum~\ref{item:desiderata-predictive}, showing a stronger correlation with generalization than any static baseline. We next proceed to validate it against Desideratum~\ref{item:desiderata-robust} (robustness under training deviations).

\section{\SHARPNESS for Model Certification}
\label{sec:exp}
The previous section demonstrated the high correlation between \sharpness and generalization, but a certification metric must also satisfy our two remaining desiderata: it must be robust against training deviations (Desideratum~\ref{item:desiderata-robust}) and efficiently computable (Desideratum~\ref{item:desiderata-efficient}). In this section, we explore these two properties. First, in Section~\ref{sec:exp:robust}, we address Desideratum~\ref{item:desiderata-robust} by showing that \sharpness consistently detects model failures that result from training deviations, including ones that cannot be detected by test accuracy or existing sharpness measures. 

Next, in Sections~\ref{sec:exp:audit} and~\ref{sec:exp:certify}, we explore efficiency in the context of two related applications: \emph{model auditing} and \emph{model certification}. We define model auditing using a single algorithm, $\mathsf{Audit}$, that takes as input a model and its training data and outputs a binary decision indicating whether or not the model is satisfactory. (Alternatively, we could imagine $\mathsf{Audit}$ outputting a confidence score or other numeric value. In Section~\ref{sec:exp:certify}, we discuss a methodology for identifying thresholds for \sharpness that verifiers could then use in predicates for producing a binary decision.) In our experiments, we show that this algorithm can be run up to $4\times$ faster with \sharpness than with test-accuracy evaluation, while also obtaining better correlation with generalization and robustness.

The $\mathsf{Audit}$ algorithm can be run by a model provider wanting to audit its own model before releasing it, but it is not appropriate for an external verifier to whom the model provider might not want to reveal the training data. Instead, we capture this setting of model \emph{certification} as a pair of algorithms, $\mathsf{Prove}$ and $\mathsf{Verify}$, where $\mathsf{Prove}$ takes as input the model and training data and outputs a proof, and $\mathsf{Verify}$ takes as input the model and the proof and again outputs either a binary decision (in keeping with the traditional model for zero knowledge and other cryptographic primitives) or a numeric score. Here, we demonstrate that proving \sharpness in zero knowledge is up to $80{,}000\times$ more efficient than proving the entire training process.

\smallskip\noindent\textbf{Testbed.}
All of our experiments are implemented in Python using PyTorch and run on a Google Colab A100 runtime with 83.5 GB of RAM and a single NVIDIA A100 GPU. For SAM training, we adopt the code provided by~\cite{latesam}. For backdoor attacks, we use the toolbox from Li et al.~\cite{li2023backdoorbox}.

\smallskip\noindent\textbf{Implementation details of sharpness measures.}
Unless specified otherwise, we set $B=16$, $T=10$, $\mathit{lr}=0.0001$, $\rho=0.05$, and use ASAM sharpness as the per-step sharpness function in \sharpness.
The static sharpness baselines we consider are also magnitude-aware. Specifically, we consider magnitude-aware worst-case sharpness, which we compute using the gradient-ascent algorithm from Jiang et al.~\cite[Algorithm 3]{jiang_fantastic_2019}; we also consider ASAM sharpness (Def.~\ref{def:approx_adapt_sharpness}). We sometimes refer to these sharpness measures as ``static'' to contrast them with the dynamic nature of our \sharpness measure.

\subsection{\SHARPNESS is More Reliable \mbox{under} Training Deviations}
\label{sec:exp:robust}

To evaluate reliability under training deviations, we run a set of \emph{distinguishability} experiments that mimic common auditing and certification settings.
Each experiment involves a benign model $M_b$ paired with a faulty model $M_f$, matched to (nearly) the same \emph{observed accuracy}—the accuracy signal the verifier would naturally inspect in that scenario (e.g., standard test accuracy, pre-quantization accuracy, or accuracy on a small test set). We then ask whether sharpness metrics can flag the faulty model when observed accuracy fails. For each pair, we evaluate observed accuracy together with two metrics computed on the \emph{clean} training set: full-dataset ASAM sharpness (Def.~\ref{def:approx_adapt_sharpness}) and \sharpness. We summarize each metric's distinguishability by its benign/faulty ratio (its value on $M_b$ divided by its value on $M_f$); a smaller ratio indicates better separation.
The failure settings we consider are:
\begin{description}[itemsep=0pt, leftmargin=1pt]
\item{\textbf{Noisy labels.}}
This setting mimics an auditor who observes similar holdout accuracy, yet one model is trained on mislabeled data. We train a VGG-16-BN $M_b$ on clean CIFAR-10 and a faulty $M_f$ on CIFAR-10 with 10\% random label corruption (each corrupted label replaced by a uniformly sampled incorrect class), stopping both runs at 90\% training accuracy. Observed accuracy is measured on the standard test set.

\item{\textbf{Overfitting to spurious features.}}
This setting mimics a scenario where two models appear equally accurate on the standard test set, but one has learned a non-robust shortcut that may fail under distribution shift.
We train a VGG-16-BN $M_b$ on clean CIFAR-10 and a faulty $M_f$ on a modified training set where, with probability $0.5$, we overlay a top-left square whose color is deterministically tied to the label. This spurious feature~\cite{geirhos2020shortcut,spur,spur2} is highly predictive during training but non-causal. Both runs are stopped at 90\% training accuracy, and observed accuracy is standard test accuracy.

\item{\textbf{Backdoors.}}
A backdoor attack~\cite{labelconsistent,sleeper,badnet,9802938} implants a hidden trigger that causes the model to behave normally on benign inputs but switch to a malicious target behavior when the trigger is present.
We train a VGG-16-BN $M_b$ for 200 epochs on the clean CIFAR-10 dataset and a faulty $M_f$ on CIFAR-10 with 1\% poisoned images, following BadNets~\cite{badnet}. Observed accuracy is clean test accuracy.

\item{\textbf{Post-quantization model quality.}} This setting mimics a scenario where two models have similar pre-quantization accuracy, but one is more robust to quantization-induced utility degradation~\cite{liu2021sharpnessaware}.
To construct two models with similar pre-quantization accuracy but different post-quantization performance, we train two VGG-16-BN models, $M_b$ with SAM and $M_f$ with standard SGD, and then perform 4-bit quantization on both models; $M_b$ sees a minimal accuracy drop ($0.920 - 0.914 = 0.006$), while $M_f$ suffers a much larger degradation ($0.918 - 0.893 = 0.025$). We report the pre-quantization accuracy as observed accuracy.

\item{\textbf{Small test sets.}} In many real-world scenarios, test sets can be unreliable, unrepresentative, or maliciously crafted. We simulate a scenario where the available test set is too small to reliably reflect true performance. We train a VGG-16-BN $M_b$ on clean CIFAR-10 and a faulty $M_f$ with $10\%$ label corruption, both for 200 epochs, and craft a 100-sample test set on which both models achieve 90\% accuracy. On the full CIFAR-10 test set, $M_f$ performs substantially worse ($75\%$ vs.\ $88\%$). Observed accuracy is the small test set accuracy. 
\end{description}

The results for all five categories are in Table~\ref{tab:distinguish}, and clearly demonstrate the limitations of static metrics. In all five scenarios, observed accuracy is (nearly) identical for the benign and faulty models, and the static sharpness ratios are much closer to 1 than the directional-sharpness ratios, indicating weaker separation.

\smallskip\noindent\textbf{Experiments on subtler and stronger deviations.}
To further test whether \sharpness remains informative beyond the settings in Table~\ref{tab:distinguish}, we run additional distinguishability experiments in Appendix~\ref{app:exp:distinguish} covering more diverse and challenging scenarios: (i) finer-grained label-noise rates; (ii) an adaptive, sharpness-aware backdoor attack~\cite{he2024sharpnessaware}; and (iii) memorization/backdoor failures on language models. Across these settings, \sharpness remains the strongest signal for benign/faulty separation.

\begin{table}[t]
\centering
\setlength{\tabcolsep}{3pt}
\begin{tabular}{lccc}
\toprule
\multirow{2}{*}{\textbf{Experiment}}
& \multicolumn{2}{c}{\textbf{Sharpness Ratio}}
& \multirow{2}{*}{\textbf{Observed Accuracy Ratio}} \\
\cmidrule(lr){2-3}
& \textbf{Directional} & \textbf{Static} & \\
\midrule
\textbf{Noisy labels}  & \textbf{0.177} & 0.991 & 1.00 \\
\textbf{Overfitting}   & \textbf{0.084} & 0.963 & 1.00 \\
\textbf{Backdoors}     & \textbf{0.061} & 0.587 & 1.00 \\
\textbf{Quantization}   & \textbf{0.240} & 0.812 & 1.00 \\
\textbf{Small test set} & \textbf{0.289} & 0.560 & 1.00 \\
\bottomrule
\end{tabular}
\caption{
  Distinguishability of faulty models under observed accuracy, static ASAM sharpness, and \sharpness. Each entry reports the benign/faulty ratio. Lower ratios indicate stronger separation. Observed accuracy denotes the natural accuracy signal available to the verifier and is experiment-specific.
}
\label{tab:distinguish}
\end{table}

\subsection{Model Quality Auditing}
\label{sec:exp:audit}
\begin{table}[t]
\centering
\begin{tabular}{l c c}
\toprule
\textbf{Metric} & {\textbf{Runtime}} &\textbf{Correlation}\\
& {\textbf{(in s)}} & \\
\midrule
\textbf{Test Accuracy} & 0.92 & --- \\
\textbf{Worst-Case Sharpness~\cite{jiang_fantastic_2019}} & 210.32 &  $\tau = 0.577$ \\
\textbf{Ours} $(B=16,T=10)$ & {\bfseries 0.45} & $\tau = 0.708$  \\
\textbf{Ours} $(B=16,T=5)$ & {\bfseries 0.29} & $\tau = 0.702$ \\
\textbf{Ours} $(B=8,T=10)$ & {\bfseries 0.39} & $\tau = 0.691$\\
\textbf{Ours} $(B=8,T=5)$ & {\bfseries 0.21} & $\tau = 0.712$ \\
\bottomrule
\end{tabular}
\caption{Runtime of evaluating different metrics and their correlation with the generalization gap.
\Sharpness is not only the most efficient metric, faster than even plain test accuracy, but also offers the strongest Kendall's $\tau$ correlation with generalization.}
\label{tab:audittime}
\end{table}

We now turn to the problem of model auditing, in which a verifier wants to measure the quality of a model given access to both the model and its training data. We show that an auditing algorithm based on standard test accuracy fails to detect certain types of flaws that \sharpness detects, and that prior sharpness metrics are too inefficient for this use case. 

\smallskip\noindent\textbf{Efficiency comparison with existing sharpness metrics.}
We measure the runtime of computing \sharpness and compare it against test accuracy and magnitude-aware worst-case sharpness (Def.~\ref{def:adapt_sharpness}) computed using the
gradient-ascent method described by Jiang et al.~\cite{jiang_fantastic_2019}.
Table~\ref{tab:audittime} shows the runtime and correlation for different metrics. When computed on a small batch ($B=8$, $T=5$), \sharpness is over $1{,}000\times$ faster than static sharpness and even $4\times$ faster than computing accuracy over the full test set (0.21 vs. 0.92 seconds). Consistent with Section~\ref{sec:flat:corr}, \sharpness also provides the strongest correlation with generalization.

\smallskip\noindent\textbf{Auditing model privacy via membership inference.}
Membership inference attacks (MIAs)~\cite{DBLP:conf/sp/ShokriSSS17} are a standard auditing tool for assessing privacy leakage: if an attacker can reliably distinguish training points (“members”) from unseen points (“non-members”), then the model reveals information about its training data. Prior work links MIA vulnerability to poor generalization and overfitting~\cite{yeom_privacy_2018,baluta2022membership}, and shows that training deviations can further amplify membership leakage~\cite{wen2024privacy,hu2022membership}. Motivated by our finding that \sharpness provides a more reliable signal of poor generalization under deviations, we ask whether it can also serve as a stronger membership score—i.e., yield better separability between members and non-members—and thus improve privacy auditing.

Most black-box MIAs use per-sample loss or confidence as the decision signal, based on the observation that memorized training examples tend to incur lower loss. We evaluate three per-example scores—loss, static ASAM sharpness, and \sharpness{}—on a VGG-16-BN model trained on CIFAR-10, using a balanced set of 1{,}000 members (training points) and 1{,}000 non-members (test points). Per-example \sharpness is computed by treating the example as the only training data and setting $B=1$. We compare attack performance using threshold-independent AUC and best-threshold accuracy (ACC). As shown in Table~\ref{tab:mia}, \sharpness achieves the highest AUC and ACC, indicating that our dynamic stability test strengthens loss-based MIA privacy auditing by improving member/non-member separability.

\begin{table}[!t]
\centering
\begin{tabular}{l cc}
\toprule
\textbf{MIA Metric} & \textbf{AUC} & \textbf{ACC}  \\
\midrule
\textbf{Loss} & 0.490 & 0.530\\
\cmidrule(lr){1-3}
\textbf{ASAM Sharpness} & 0.532 & 0.539  \\
\textbf{\SHARPNESS (Ours)} & \textbf{0.556} & \textbf{0.570}  \\
\bottomrule
\end{tabular}
\caption{MIA performance when using different per-sample scores as the attack signal.
Higher AUC indicates better separability.
}
\label{tab:mia}
\end{table}

\subsection{Model Quality Certification}
\label{sec:exp:certify}

We finally turn our attention to the problem of model certification, in which a prover wants to convince a verifier of the quality of a model to which the verifier has access (but, crucially, the verifier does not have access to the training data). Given that we have demonstrated the strong correlation between our \sharpness metric and model generalization (and thus quality), we show how to achieve this by generating a zero-knowledge proof (ZKP) of \sharpness.

\begin{figure}[t!]
    \centering
\ifarxiv
\includegraphics[width=0.5\linewidth]{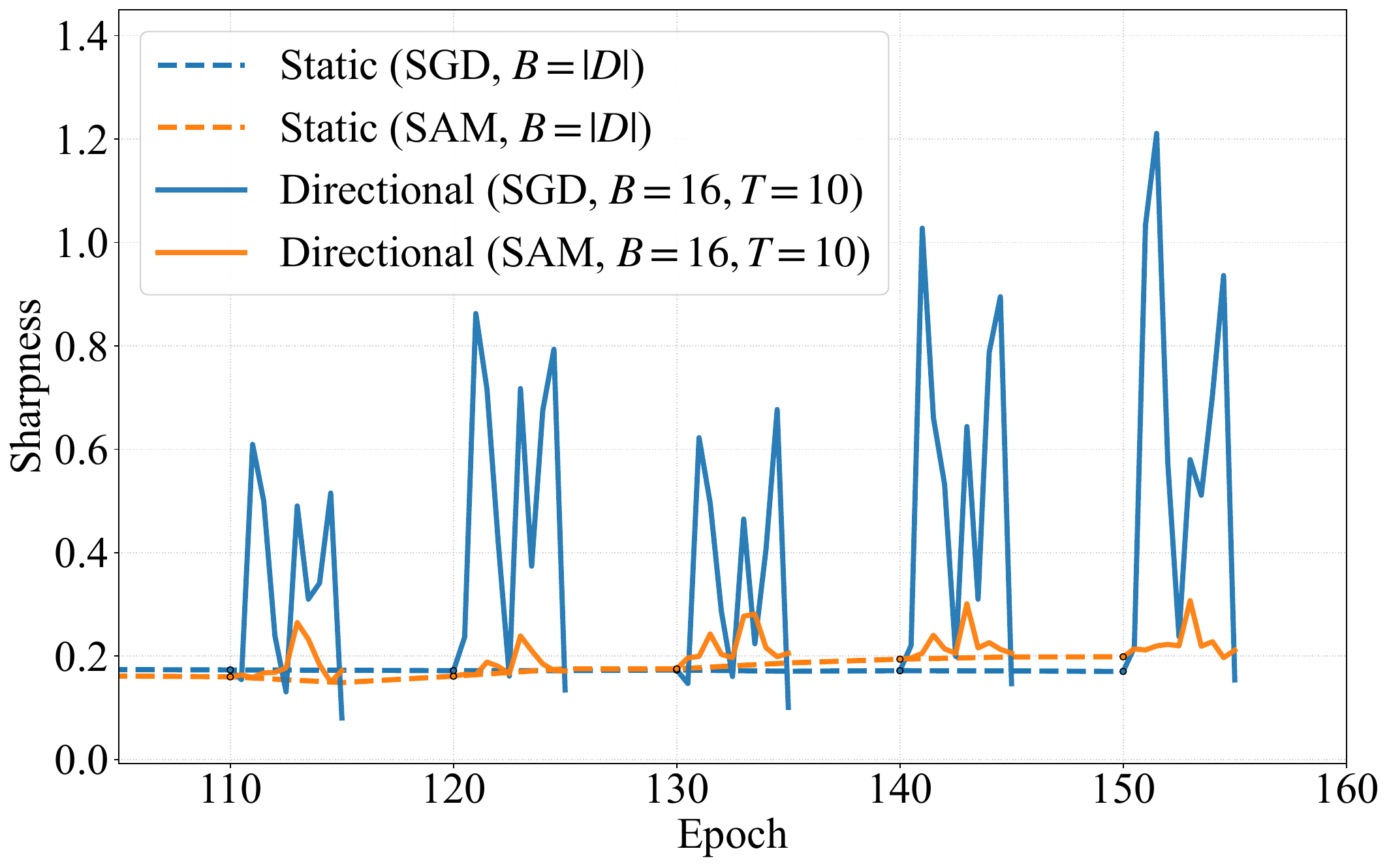}
\else
\includegraphics[width=1\linewidth]{Figures/seed0_bs32_branched_warped_adjusted.pdf}
\fi
\caption{Static ASAM sharpness vs. \sharpness probe trajectories on models trained with SGD and SAM.
Static sharpness incorrectly suggests the SGD model generalizes better at epochs 130, 140, and 150, while \sharpness trajectories reveal the SAM model is more stable under perturbations.
}

\label{fig:sam_erm}
\end{figure}

\begin{table}[!t]
\centering
\begin{tabular}{c ccc}
\toprule
\multirow{2}{*}{\textbf{Models}} & \multicolumn{2}{c}{\textbf{Sharpness}} & \multirow{2}{*}{\textbf{Generalization Gap}} \\
\cmidrule(lr){2-3}
 & \textbf{Ours}& \textbf{Static ASAM}  & \\
\midrule
 \textbf{SAM} &\textbf{1.123}  & \textbf{0.195} & \textbf{0.205} \\
\textbf{SGD} &5.711  & 0.170 & 0.273 \\

\bottomrule
\end{tabular}
\caption{Static ASAM sharpness vs. \sharpness values on two models trained with SAM and SGD.
Static sharpness incorrectly suggests the SGD model has better generalization, while \sharpness correctly identifies the SAM model as the better-generalizing one.} 

\label{tab:samsgd}
\end{table}
\smallskip\noindent\textbf{Effectiveness of \sharpness.}
We first demonstrate the effectiveness of \sharpness compared to existing static sharpness metrics in model certification. In particular, we build on our previous results by considering the specific use case in which a malicious model trainer wants to convince a verifier that a model was trained using SAM (which is expensive) but, in reality, trained the model using standard SGD (which is much cheaper). To simulate this scenario, we train two VGG-16-BN models on CIFAR-10 for 150 epochs. One model was trained with standard SGD, representing the malicious trainer's model, while the other was honestly trained with SAM. After training, the SAM model has a lower generalization gap (0.205) than the SGD model (0.273) and thus better generalization.

Figure~\ref{fig:sam_erm} plots the trajectory of both sharpness metrics over the last $50$ epochs of training. The static ASAM sharpness metric (Def.~\ref{def:approx_adapt_sharpness}, dashed lines) proves to be an unreliable signal, showing both models converging to similar values and incorrectly suggesting the SGD model has better generalization (Table~\ref{tab:samsgd}). In contrast, \sharpness (solid lines) is effective at distinguishing SAM and SGD models at all epochs. The {SAM model} (orange lines) consistently displays low and stable sharpness, indicating it resides in a uniformly flat basin. The {SGD model} (blue lines) exhibits high and erratic sharpness spikes. This shows that while its full-dataset sharpness may appear low, it is stuck in a sharp basin that is immediately revealed by \sharpness.

\smallskip\noindent\textbf{Efficiency of \sharpness.}
Now that we have established \sharpness as more effective for model certification, we begin evaluating the end-to-end performance of proving \sharpness in zero knowledge to demonstrate its efficiency. We compare this cost against a \emph{zero-knowledge proof of training (zkPoT)}~\cite{kaizen,garg2023experimenting}.

To obtain realistic and concrete cost estimates, we extrapolate from the reported results of Kaizen~\cite{kaizen}. Kaizen proves mini-batch gradient-descent training using an optimized GKR-style proof for each gradient update. Its instantiation uses 64-bit fixed-point encoding with 32 fractional bits over $\mathbb{F}_{p^2}$, where $p=2^{61}-1$. Mini-batches are sampled by applying a public pseudorandom permutation to the committed dataset indices and taking consecutive blocks of size $B$. Since both the zkPoT baseline and our ZKP of directional sharpness use the same Kaizen instantiation with parameters chosen for $\lambda=100$-bit security, both inherit a knowledge-soundness error of at most $2^{-100}$ for the corresponding statements. We refer to~\cite{kaizen} for their full protocol and implementation details.

Let $P_B^{\text{SGD}}$ be the cost of proving one SGD iteration on batch size $B$ with prover runtime and proof size taken from~\cite{kaizen}. Since SAM requires twice the computation per step versus SGD~\cite{sam}, we set $P_B^{\text{SAM}} \approx 2 \cdot P_B^{\text{SGD}}$.
To avoid over-penalizing our baseline, we adopt the finding that 1--5 epochs of SAM at the late stage achieve comparable generalization to full SAM training~\cite{latesam} and consider a ``LATE-SAM'' baseline that proves 145 SGD epochs followed by 5 SAM epochs with batch size $B=16$. The total cost for the LATE-SAM baseline is therefore $\frac{155|D|}{B} \cdot P_B^{\text{SGD}}$.
Our ZKP of \sharpness proves $T$ SAM steps plus a small circuit for the fluctuation function. Since the cost of the fluctuation function is small compared with $P_B^{\text{SAM}}$, we estimate our total cost as $(2T) \cdot P_B^{\text{SGD}}$.

Table~\ref{tab:zke2e} reports these cost estimates. A ZKP of \sharpness (running for $T=5$ steps with a $B=8$ batch) is up to $80{,}000\times$ faster than proving the full 150-epoch LATE-SAM training run: a ZKP of LATE-SAM training for VGG-11 would take an estimated $118{,}671$ hours (over 13 years), while our ZKP of \sharpness takes under 90 minutes.

\begin{table}[t]\centering
\begin{tabular}
{@{\hskip0pt}c cc@{\hskip1pt}c cc@{\hskip1pt}}
\toprule
\multirow{4}{*}{\textbf{Models}} & \multicolumn{3}{c}{\textbf{ZKP of \SHARPNESS}}  & \multicolumn{2}{c}{\textbf{ZKP of Training}} \\
& \multicolumn{3}{c}{\textbf{$(B=8,T=5)$}} &  \multicolumn{2}{c}{\textbf{$(B=16)$}} \\
\cmidrule(lr){2-4}\cmidrule(lr){5-6}
& \textbf{Time} & \textbf{Comm.} & \textbf{Speedup}
& \textbf{Time} & \textbf{Comm.}    \\
& \textbf{\footnotesize (min.)} & \textbf{\footnotesize(MB)} & \textbf{\footnotesize ($\times$)} & \textbf{\footnotesize (hr.)} & \textbf{\footnotesize(GB)}  \\
\midrule
\textbf{LeNet}    & 23.63 & 9.85 & $\mathbf{66{,}063}$ & ${26{,}016}$ & ${471.64}$\\
\textbf{AlexNet}    & 55.43 & 11.59 & $\mathbf{69{,}091}$ & ${63{,}830}$ & ${579.72}$\\
\textbf{VGG-11}   & 88.87 & 15.28 & $\mathbf{80{,}119}$ & $118{,}671$ & $751.59$\\
\bottomrule
\end{tabular}
\caption{Cost estimates for ZKPs of \sharpness vs. ZKPs of LATE-SAM training. For the zkPoT baseline, all models are trained for 145 epochs using SGD and switched to SAM for the final 5 epochs. $|\data|=50{,}000$.}
\label{tab:zke2e}
\end{table}

\smallskip\noindent\textbf{Certification decision based on \sharpness.}
Our experiment shows that \sharpness provides a robust signal for distinguishing high-quality models from low-quality ones. In practice, however, a verifier must still choose a threshold or decision rule to convert this signal into a binary ``accept/reject'' decision.

We note that this threshold selection is a limitation of any metric-based certification approach, since the true generalization gap is unobservable. Thus, the goal is not to find a universal constant, but rather a metric reliable enough that practical calibration procedures can be effective. Our theoretical and empirical results show that \sharpness is well suited for this purpose: theoretically, it provides a principled measure of instability that is robust under training deviations (Section~\ref{sec:flat:math} and Section~\ref{sec:flat:deviation}); empirically, it exhibits stronger correlation with generalization (Section~\ref{sec:flat:corr}) and wider separation margins between benign and faulty models (Table~\ref{tab:distinguish}) compared to existing metrics.

These properties enable a straightforward calibration procedure: (i) construct a small calibration set of models of known quality for the task family (e.g., a verifier may already have baseline models trained with SGD or SAM); (ii) select a threshold $\tau$ that separates these models based on an acceptable risk level; (iii) apply $\tau$ to evaluate new models within the same task family. In certification settings where a binary decision is not required, verifiers can also treat the \sharpness value directly as a continuous confidence score, with lower values indicating higher model quality.

\section{Conclusion}

Model quality certification is essential for providing transparency and accountability in machine learning systems. In this work, we use generalization as a proxy for overall model quality and develop a new, efficient methodology for its certification. Our central contribution is a metric we call \sharpness. This measure strongly correlates with generalization, is reliable under existing training deviations, and is efficiently computable. These properties make \sharpness a suitable ingredient in ML model certification.
Therefore, \sharpness 
can play an important role in mitigating attacks in adversarial machine learning. Using \sharpness to devise rigorous defenses against adversarial models, and incorporating these defenses into existing systems, is a natural continuation of this research.

\appendix
\section*{Acknowledgments}
Gefei Tan is supported by a Google PhD Fellowship.
\bibliographystyle{plainurl} 
\bibliography{all}

\section*{Appendix}

\section{Additional Experiments}
\label{app:exp:distinguish}
\begin{table}[!t]
\centering
\begin{tabular}{lccc}
\toprule
\multirow{2}{*}{\textbf{Experiment}}
& \multicolumn{2}{c}{\textbf{Sharpness Ratio}}
& \multirow{2}{*}{\textbf{Observed Acc. Ratio}} \\
\cmidrule(lr){2-3}
& \textbf{Directional} & \textbf{Static} & \\
\midrule
\textbf{SAPA} & \textbf{0.016} & 0.217 & 1.000 \\
\textbf{LM-Mem}     & \textbf{0.083} & 0.664 & 1.003 \\
\textbf{LM-BD}      & \textbf{0.053} & 0.630 & 1.005 \\
\textbf{Noise (1\%)}   & \textbf{0.122} & 0.676 & 0.995 \\
\textbf{Noise (2.5\%)} & \textbf{0.086} & 0.833 & 0.994 \\
\textbf{Noise (5\%) }   & \textbf{0.079} & 0.546 & 0.998 \\
\textbf{Noise (7.5\%)} & \textbf{0.081} & 0.418 & 0.998 \\
\textbf{ResNet-BD}  & \textbf{0.253} & 0.972& 1.000 \\
\bottomrule
\end{tabular}
\caption{
Results of additional distinguishability experiments. Each entry reports the benign/faulty ratio; lower values indicate better separation. \Sharpness provides the strongest separation.
}
\label{tab:additional_distinguish}
\end{table}
We provide additional distinguishability experiments extending Section~\ref{sec:exp:robust}, evaluating whether \sharpness remains informative under subtler, more challenging deviations and on non-vision models.
We follow the same evaluation protocol as in Section~\ref{sec:exp:robust}: for each setting, we train benign/faulty model pairs and ask whether sharpness metrics can separate them when observed accuracy fails. Table~\ref{tab:additional_distinguish} reports the benign/faulty ratios for the same metrics as in Section~\ref{sec:exp:robust}.

\smallskip\noindent\textbf{Backdoor on ResNet-18 (ResNet-BD).}
We repeat the BadNets backdoor experiment on ResNet-18. The faulty model is trained on CIFAR-10 with 1\% poisoned images. Both benign and faulty models are trained for 160 epochs. The observed accuracy reported is the clean test accuracy.

\smallskip\noindent\textbf{Varying label noise (Noise).}
We repeat the label-noise experiment with varying noise rates to evaluate the sensitivity of \sharpness to weaker forms of data-quality degradation. The observed accuracy reported is the clean test accuracy.

\smallskip\noindent\textbf{Sharpness-aware poisoning attack (SAPA).} SAPA~\cite{he2024sharpnessaware} uses SAM-style perturbations during poison generation to optimize the adversarial objective, making the poison more robust to retraining uncertainty. We train a benign ResNet-18 on clean CIFAR-10 using standard SGD, and a faulty ResNet-18 using SAPA's hidden-trigger backdoor attack with $1\%$ of the training set poisoned, perturbation budget $\epsilon=16/255$, and sharpness radius $\rho=0.05$. Both models are trained for 160 epochs. The observed accuracy is the clean test accuracy.

\smallskip\noindent\textbf{Backdoor (LM-BD) and memorization (LM-Mem) on language models.}
We extend our distinguishability tests to transformer language models using TinyMem~\cite{sakarvadia2025mitigating}, a suite of lightweight GPT-2-style models. Using the 4-layer multiplicative architecture, we train a benign model $M_b$ on clean data and a faulty model $M_f$ on data poisoned with either the \textsf{backdoor} artifact~\cite[Def.\ 2.3]{sakarvadia2025mitigating}
or the \textsf{noise} artifact~\cite[Def.\ 2.2]{sakarvadia2025mitigating}, which inject a trigger pattern and a random perturbation, respectively, into a small fraction of the
training sequences (see~\cite{sakarvadia2025mitigating} for details). In both settings, $M_f$ memorizes the injected artifacts while matching the observed (clean-test) accuracy of $M_b$: the backdoored $M_f$ attains $99.3\%$ backdoor
memorization at $96.92\%$ observed accuracy (vs.\ $97.41\%$ for $M_b$), and the noise-trained $M_f$ memorizes $37.4\%$ of the artifacts at $97.11\%$ observed accuracy (vs.\ $97.40\%$ for $M_b$).

\section{{Proofs of Lemmas and Theorems}}
We first establish two direct consequences of our assumptions. Under Assumption~\ref{as:PL}, we have
\begin{equation}
2\mu \LLL(\vecw)\le \|g(\vecw)\|_2^2 \le 2L_s\LLL(\vecw),
\label{eq:pl-smooth-sandwich}
\end{equation}
where the upper bound is a standard result from smoothness~\cite[Lemma~1]{khaled2022better} with $\inf_{\vecw}\LLL(\vecw)=0$.

Additionally, under Assumption~\ref{as:noise}, the gradient noise for a single sample $i$ satisfies
$\mathbb{E}_{i}\|\zeta_i^1(\vecw)\|_2^2\le \sigma^2\LLL(\vecw).$
For a mini-batch $\xi$ of size $B$ sampled uniformly without replacement,
we have
\[
\mathbb{E}_\xi\|\zeta^B(\vecw,\xi)\|_2^2
=
\frac{N-B}{B(N-1)}
\mathbb{E}_{i}\|\zeta_i^1(\vecw)\|_2^2.
\]
Since $(N-B)/(N-1)\le 1$, we obtain
\begin{equation}
\mathbb{E}_\xi\|\zeta^B(\vecw,\xi)\|_2^2
\le \frac{\sigma^2}{B}\LLL(\vecw).
\label{eq:batch-noise}
\end{equation}

\subsection{Proof of Lemma~\ref{thm:linear-sandwich}}
\label{app:proof:sandwich}
\begin{proof}
For $p=2$, the SAM sharpness on mini-batch $\xi_t$ is
\[
s_t
= \rho\|\nabla \LLL_{\xi_t}(\vecw_t)\|_2
= \rho\|g(\vecw_t)+\zeta_t\|_2,
\]
where $\zeta_t:=g_{\xi_t}(\vecw_t)-g(\vecw_t)$. Conditioning on $\vecw_t$, and using the fact that the mini-batch gradient is unbiased, we have
\[
\mathbb{E}[s_t^2\mid \vecw_t]
= \rho^2\left(\|g(\vecw_t)\|_2^2
+\mathbb{E}[\|\zeta_t\|_2^2\mid \vecw_t]\right).
\]
For the lower bound, drop the nonnegative noise term and apply the PL lower bound
in Eq.~\ref{eq:pl-smooth-sandwich}:
\[
\mathbb{E}[s_t^2\mid \vecw_t]
\ge \rho^2\|g(\vecw_t)\|_2^2
\ge 2\rho^2\mu\,\LLL(\vecw_t).
\]
For the upper bound, apply the upper bound in Eq.~\ref{eq:pl-smooth-sandwich}
and the mini-batch noise bound Eq.~\ref{eq:batch-noise}:
\[
\mathbb{E}[s_t^2\mid \vecw_t]
\le \rho^2\left(2L_s+\frac{\sigma^2}{B}\right)\LLL(\vecw_t).
\]
This completes the proof.
\end{proof}

\subsection{Proof of Theorem~\ref{thm:sharpness-dynamics}}
\label{app:proof:dynamic}
\begin{proof}
Taking expectations in Lemma~\ref{thm:linear-sandwich} and using the tower rule
gives
\begin{equation}
\label{small-sandwich}
\underline{\kappa}_\rho\,\mathbb{E}[\LLL(\vecw_t)]
\le
\mathbb{E}[s_t^2]
\le
\overline{\kappa}_\rho\,\mathbb{E}[\LLL(\vecw_t)].
\end{equation}

\noindent\textbf{\textit{(a) Stability implies bounded sharpness.}}
SAM-stability gives
$\mathbb{E}[\LLL(\vecw_t)]\le C\LLL(\vecw_0)$ for all $t$. Substituting into the
upper bound in Eq.~\ref{small-sandwich} gives
\[
\mathbb{E}[s_t^2]
\le
\overline{\kappa}_\rho C\LLL(\vecw_0).
\]

\noindent\textbf{\textit{(b) Exponential sharpness growth certifies instability.}}
The upper side of the sandwich bound in Eq.~\ref{small-sandwich} implies
\[
\mathbb{E}[\LLL(\vecw_t)]
\ge
\frac{1}{\overline{\kappa}_\rho}\mathbb{E}[s_t^2].
\]
Therefore, if $\mathbb{E}[s_t^2]\ge K \alpha^t\LLL(\vecw_0)$ for all $t\ge t_0$, then
\[
\mathbb{E}[\LLL(\vecw_t)]
\ge
\frac{K}{\overline{\kappa}_\rho}\alpha^t\LLL(\vecw_0).
\]
Since $\alpha>1$ and $\LLL(\vecw_0)>0$, this loss trajectory is unbounded relative
to $\LLL(\vecw_0)$ and violates the condition in Def.~\ref{def:sam_stability}. Hence $\vecw^*$ is SAM-unstable.
\end{proof}

\subsection{Proof of Theorem~\ref{thm:detection-gap}}
\label{app:detection-gap-proof}
We first prove the following lemma for mini-batch gradients.

\begin{lemma}
\label{lem:mb-decomp}
Let $\xi$ be a mini-batch of size $B$ sampled uniformly without replacement from
a dataset of size $N$. Write $S_g^2:=\frac1N\sum_{i=1}^N\|g_i(\vecw)\|_2^2$ and define $c_g(\vecw):=\|g(\vecw)\|_2^2/S_g^2$ (Def.~\ref{def:grad-coh}). Assuming $S_g^2>0$,
\[
\mathbb{E}_\xi\|g_\xi(\vecw)\|_2^2
=
\|g(\vecw)\|_2^2
+
\frac{N-B}{B(N-1)}(1-c_g(\vecw))S_g^2.
\]
\end{lemma}

\begin{proof}
Expanding the squared norm and using
$\Pr[i\in\xi]=B/N$ and
$\Pr[i,j\in\xi]=B(B-1)/(N(N-1))$ for $i\ne j$,
\begin{align*}
\mathbb{E}_\xi\|g_\xi\|_2^2
&=
\frac{1}{B^2}
\left(
\frac{B}{N}\sum_i \|g_i\|_2^2
+
\frac{B(B-1)}{N(N-1)}
\sum_{i\ne j}\langle g_i,g_j\rangle
\right)  \\
&=
\frac1BS_g^2
+
\frac{B-1}{B}\cdot
\frac{1}{N(N-1)}
\sum_{i\ne j}\langle g_i,g_j\rangle .
\end{align*}
Since
\[
\sum_{i\ne j}\langle g_i,g_j\rangle
=
N^2\|g\|_2^2-NS_g^2,
\]
substitution gives
\[
\mathbb{E}_\xi\|g_\xi\|_2^2
=
\|g\|_2^2+
\frac{N-B}{B(N-1)}(S_g^2-\|g\|_2^2).
\]
Plugging in $S_g^2-\|g\|_2^2=(1-c_g(\vecw))S_g^2$ yields the lemma.
\end{proof}

\begin{proof}[Proof of Theorem~\ref{thm:detection-gap}]
For $p=q=2$,
\[
\mathsf{S}=\rho\|g(\vecw)\|_2,
\qquad
\mathsf{S}_\xi=\rho\|g_\xi(\vecw)\|_2,
\qquad
\mathsf{S}_{\mathrm{rms}}^2=\mathbb{E}_\xi[\mathsf{S}_\xi^2].
\]
Therefore,
\[
\mathsf{S}_{\mathrm{rms}}^2-\mathsf{S}^2
=
\mathbb{E}_\xi[\mathsf{S}_\xi^2]-\mathsf{S}^2
=
\rho^2\left(\mathbb{E}_\xi\|g_\xi(\vecw)\|_2^2-\|g(\vecw)\|_2^2\right).
\]
Applying Lemma~\ref{lem:mb-decomp} yields Eq.~\ref{eq:gap-diff}:
\[
\mathsf{S}_{\mathrm{rms}}^2 - \mathsf{S}^2
=
\rho^2\cdot \frac{N-B}{B(N-1)} \cdot \bigl(1-c_g(\vecw)\bigr)\cdot S_g^2.
\]

If $B\ll N$, then $\frac{N-B}{B(N-1)}\approx \frac1B$. If also $c_g(\vecw)\ll 1$, then $1-c_g(\vecw)\approx 1$, and hence we have Eq.~\ref{eq:dom}:
\[
\mathsf{S}_{\mathrm{rms}}^2-\mathsf{S}^2
\approx
\rho^2\frac1BS_g^2
=
\rho^2\frac1B\cdot
\frac1N\sum_{i=1}^N\|g_i(\vecw)\|_2^2.
\]

Finally, if $\|g(\vecw)\|_2\le \varepsilon$, then
$\mathsf{S}^2\le \rho^2\varepsilon^2$. Therefore, whenever
$\varepsilon^2\ll S_g^2/B$, we have
\[
\mathsf{S}^2
\ll
\rho^2 S_g^2/B
\approx
\mathsf{S}_{\mathrm{rms}}^2-\mathsf{S}^2.
\]
Hence $\mathsf{S}^2\ll \mathsf{S}_{\mathrm{rms}}^2$ and
$\mathsf{S}\ll \mathsf{S}_{\mathrm{rms}}$.
\end{proof}

\end{document}